\documentclass{article}

\usepackage[final]{neurips_2021}

\usepackage[utf8]{inputenc} 
\usepackage[T1]{fontenc}    
\usepackage{hyperref}       
\usepackage{url}            
\usepackage{booktabs}       
\usepackage{amsfonts}       
\usepackage{nicefrac}       
\usepackage{microtype}      
\usepackage{xcolor}         
\usepackage{graphicx,subfigure}       

\usepackage{etoolbox}
\newtoggle{long-version}
\togglefalse{long-version}

\usepackage{amsmath}
\usepackage{amsthm}
\newtheorem{theorem}{Theorem}[section]
\newtheorem{lemma}[theorem]{Lemma}
\newtheorem{assumption}[theorem]{Assumption}

\theoremstyle{definition}

\theoremstyle{remark}

\title{The Impact of Batch Learning in Stochastic Bandits}

\author{%
  Danil Provodin$^{1,2}$, Pratik Gajane$^1$, Mykola Pechenizkiy$^{1,4}$, Maurits Kaptein$^{2,3}$ \\
  $^1$Eindhoven University of Technology, Eindhoven, The Netherlands\\
  $^2$Jheronimus Academy of Data Science, ‘s-Hertogenbosch, The Netherlands \\
  $^3$Tilburg University, Tilburg, The Netherlands \\
  $^4$University in Jyväskylä, Jyväskylä, Finland \\
  \texttt{ \{d.provodin,p.gajane,m.pechenizkiy\}@tue.nl, M.C.Kaptein@tilburguniversity.edu} \\
}

\begin{document}

\newcommand{\regret}{R}
\newcommand{\horizon}{n}
\newcommand{\horizonn}{n_1}
\newcommand{\horizonnn}{n_2}
\newcommand{\reward}{X}
\newcommand{\vecreward}{\Bar{X}}
\newcommand{\actset}{\mathcal{A}}
\newcommand{\Action}{A}
\newcommand{\action}{a}
\newcommand{\timeidx}{t}
\newcommand{\timeidxx}{s}
\newcommand{\totalreward}{S}
\newcommand{\policy}{\pi}
\newcommand{\dumnpolicy}{\pi^{\prime}}
\newcommand{\armscount}{N}
\newcommand{\rewardof}{r}
\newcommand{\numofactions}{K}
\newcommand{\armval}{\mu}
\newcommand{\bestarm}{\mu^{\ast}}
\newcommand{\bestarmidx}{a^{\ast}}
\newcommand{\suboptgap}{\Delta}
\newcommand{\numbatches}{M}
\newcommand{\batchsize}{b}
\newcommand{\grid}{\mathcal{T}}
\newcommand{\batchidx}{j}
\newcommand{\batchidxoftimeidx}{\batchidx(\timeidx)}
\newcommand{\batchpolicy}{\policy^{\batchsize}}
\newcommand{\history}{H}
\newcommand{\realnumbers}{\mathbb{R}}
\newcommand{\historyset}{\mathcal{H}}
\newcommand{\setofdist}[1]{\mathcal{M}_1{#1}}
\newcommand{\probmeasure}{\mathbb{P}_{\nu, \pi}}
\newcommand{\env}{\nu}
\newcommand{\avpolicy}[1]{\Bar{\policy}_#1}
\newcommand{\avbatchpolicy}[1]{\Bar{\policy}^{\batchsize}_#1}
\newcommand{\avdumnpolicy}[1]{\Bar{\policy}^{\prime}_#1}

\maketitle

\begin{abstract}
We consider a special case of bandit problems, namely batched bandits. Motivated by natural restrictions of recommender systems and e-commerce platforms, we assume that a learning agent observes responses batched in groups over a certain time period. Unlike previous work, we consider a more practically relevant batch-centric scenario of batch learning. We provide a policy-agnostic regret analysis and demonstrate upper and lower bounds for the regret of a candidate policy. Our main theoretical results show that the impact of batch learning can be measured in terms of online behavior. Finally, we demonstrate the consistency of theoretical results by conducting empirical experiments and reflect on the optimal batch size choice.
\end{abstract}

\section{Introduction}

\textbf{Background and motivation.} The multi-armed bandit (MAB) problem is one of the central topics of modern literature on sequential decision making, which aims to determine policies that maximize the expected outcome. These policies are often learned either \textit{online} (sequentially) (see, e.g., \citep{LinUCB_2010, NIPS2017_28dd2c79, dimakopoulou2018estimation}) or \textit{offline} (statically) (see, e.g., \citep{swaminathan15a, zhou2018offline, joachims2018deep, athey2020policy}. In online problems, the agent learns through sequential interaction with the environment adjusting the behavior for every single response. In offline learning on the other hand, the agent learns from a fixed historical data without the possibility to interact with the environment and, therefore, the goal of the agent is to maximally exploit the static data to determine the best policy. However, neither setting provides a close approximation of the underlying reality in many cases. While offline setting is simply not conducive to sequential learning, applicability of online learning is often curtailed by limitations of practical applications. For example, in recommender systems and ad placement engines, treating users one at a time can become a formidable computational burden; in online marketing and clinical trials, environments designs (campaigns/trials) and the presence of delayed feedback result in treating patients/customers organized into groups. In all of these applied cases, it is infeasible to learn one-by-one due to computational complexity or the impact of delay.

Because of the practical restrictions described above, we consider \textit{sequential batch} learning in bandit problems -- sequential interaction with the environment when responses are grouped in batches and observed by the agent only at the end of each batch. Broadly speaking, sequential batch learning is a more generalized way of learning which covers both offline and online settings as special cases bringing together their advantages. 
\begin{itemize}
    \item  Unlike offline learning, sequential batch learning retains the sequential nature of the problem.
    \item Unlike online learning, it is often appealing to implement  batch learning in large scale bandit problems as
        \begin{itemize}
            \item it does not require much engineering efforts and resources, as experimental control over the system is not necessary; and
            \item it does not require resources to make the feedback loop shorter in time-delayed bandit problems.
        \end{itemize}
\end{itemize}

In many application domains, batched feedback is an intrinsic characteristic of the problem \citep{marketing_DP, NIPS2011_e53a0a29, marketing_bandit, Hill_2017}. Although the batch setting is not ubiquitous in the traditional stochastic MAB formulation, there have been several attempts to extend that framework. A more restrictive version of the batch problem, where the agent can choose each arm at most once in a single batch, was studied by \citet{Anantharam1987AsymptoticallyEA}. \citet{pmlr-v51-jun16} generalized the batch problem by introducing budget (how many times each arm can be chosen). However, unlike our case, the authors considered the problem of identifying top-$k$ arms with high probability. The origin of our problem formulation can be traced back to \citet{Perchet_2016}, who proposed an explicit batched algorithm based on explore-then-commit policy for a two-armed batch bandit problem with a relatively small number of batches and explore its upper and lower regret bounds, giving rise to a rich line of work \citet{gao2019batched, esfandiari2020regret, han2020sequential}. The problem of batched bandits is also related to learning from delayed or/and aggregated feedback (see, e.g., \citep{joulani2013online, vernade2017stochastic, pikeburke2018bandits}), since the decision maker is not able to observe rewards in the interim of a batch. Although these are two similar problems, learning from delayed feedback deals with exogenous delays, whereas feedback delays in sequential batch learning are endogenous and arise as a consequence of the batch constraints \citep{han2020sequential}.

Unfortunately, a comprehensive understanding of the effects of the batch setting is still missing. Withdrawing the assumptions of online learning that has dominated much of the MAB literature raises fundamental questions as how to benchmark performance of candidate policies, and how one should choose the batch size for a given policy in order to achieve the rate-optimal regret bounds. As a consequence, it is now frequently a case in practice when the batch size is chosen for the purpose of computational accessibility rather than statistical evidence \citep{NIPS2011_e53a0a29, Hill_2017}. Moreover, while the asymptotics of batched policies are known (see, e.g., \citet{Auer_2010, cesabianchi2013online}),
the relatively small horizon performance of batch policies is much less understood while simultaneously being much more practically relevant. Thus, in this work, we make a significant step in these directions by providing a systematic study of the sequential batch learning problem.

\textbf{Main contribution.} In this paper, we develop a systematic approach to address the challenges mentioned above. First, we formulate a more practically relevant batch-centric problem.
The second dimension of contribution lies in the analysis domain. We provide a refined theoretical analysis and establish upper and lower bounds on the performance for an arbitrary candidate policy. On the modeling side, we demonstrate the validity of the theoretical results by conducting empirical experiments and reflect on the choice of the optimal batch size. Lastly, we present directions for future work.

\section{Problem formulation}
\label{formulation}
We consider a variant of stochastic bandits, which we call \textit{sequential batch learning}. In this problem, the decision-maker (agent) has to make a sequence of decisions (actions), and for each decision it receives a stochastic reward. \footnote{We use the same problem formulation for online setting as described in \citep{lattimore_szepesvari_2020}.} More formally, given a finite set of arms $\actset = \{1, ..., \numofactions\}$, an environment $\env = (P_\action : \action \in \actset)$ ($P_\action$ is the distribution of rewards for action $\action$), and a time horizon $\horizon$, at each time step $\timeidx \in \{1,2, \dots, \horizon\}$, the agent chooses an action $\Action_{\timeidx} \in \actset$ and receives a reward $\reward_{\timeidx} \sim P_{\Action_\timeidx}$. 
Note that we only require the existence of expectation from distributions $P_\action$ for $\action \in \actset$.

\subsection{Batch policies, performance, and regret}

The goal of the agent is to maximize the total reward $\totalreward_{\horizon} = \sum_{\timeidx=1}^{\horizon} \reward_{\timeidx}$. Let $\armval_\action (\env) = \int_{\realnumbers} x d P_\action(x)$ be the expected reward of action $\action \in \actset$ and $\bestarm (\env) = \max_{\action \in \actset} \armval_\action (\env)$ be the best action. We consider a classical performance measure of the agent -- regret,  which is the difference between the players' rewards after $\horizon$ rounds and the best reward possible given the strategy of playing a single arm:

\begin{equation*}
\label{eq:regret definition}
    \regret_\horizon (\policy, \env) = \horizon \mu^*(\env) - \mathbb{E} \big [ \totalreward_{\horizon} \big ].
\end{equation*}
    
Throughout the work we assume that $\env$ is an arbitrary but fixed environment, therefore we will often drop the dependence on $\env$ in various quantities.
    
A policy is a rule that describes how the actions should be taken in light of the past. Here, the past at time step $\timeidx > 0$ is defined as

\begin{equation*}
    \history_{\timeidx} = (\Action_1, \reward_1, ... , \Action_{\timeidx}, \reward_{\timeidx}) \in (\actset \times \realnumbers)^{\timeidx} \equiv \historyset_{\timeidx}
\end{equation*}

which is the sequence of action-reward pairs leading up to the state of the process at the previous time step $\timeidx-1$. Note that $\history_0 = \emptyset$. Let $\setofdist(X)$ be a set of all probability distributions over a finite set $X$. As such, a policy is a finite sequence $\policy = (\policy_{\timeidx})_{1 \leq \timeidx \leq \horizon}$ of maps of histories to distributions over actions, formally, $\policy_\timeidx : \historyset_{\timeidx-1} \xrightarrow{} \setofdist(\actset)$. Intuitively, following a policy $\policy$ means that in time step $\timeidx>0$, the distribution of the action $\Action_\timeidx$ to be chosen for that timestep is $\policy_{\timeidx}(\history_{\timeidx-1})$: the probability that $\Action_{\timeidx} = a$ is $\policy_{\timeidx}(\history_{\timeidx-1})(\action)$. Since writing $\policy_{\timeidx}(\history_{\timeidx-1})(\action)$ is quite cumbersome, we abuse notation and will write $\policy_{\timeidx}(\action | \history_{\timeidx-1})$. Thus, when following a policy $\policy$, in time step $\timeidx$ we get that

\begin{equation*}
    \probmeasure(\Action_{\timeidx} = \action | \history_{\timeidx-1}) = \policy_{\timeidx}(\action | \history_{\timeidx-1}).
\end{equation*}

In contrast to conventional approaches that require the reward to be observable after each choice of the arm, our setting assumes only that rewards are released after $b$ consecutive iterations. Denote by $\grid = \timeidx_1,...,\timeidx_{\numbatches}$ a grid, which is a division of the time horizon $\horizon$ to $\numbatches$ batches of equal size $\batchsize$,  $1 = \timeidx_1 < ... < \timeidx_{\numbatches} = \horizon, \timeidx_{\batchidx} - \timeidx_{\batchidx-1} = \batchsize$ for all $\batchidx=1,..., \numbatches$. Without loss of generality we assume that $\horizon=\batchsize \numbatches$, otherwise we can take $\horizon:= \left\lfloor \frac{\horizon}{\batchsize} \right\rfloor \batchsize $. Recall that in the batch setting the agent receives the rewards after a batch ends, meaning that the agent operates with the same amount of information within a single batch. For simplicity, we assume that as long as the history remains the same the decision rule does not change as well. Note that this assumption does not impose any significant restrictions. Indeed, instead of applying a policy once, one can always do it $\batchsize$ times until the history updates. Thus, a batch policy is also a finite sequence of $\policy = (\policy_{\timeidx})_{1 \leq \timeidx \leq \horizon}$ of maps of histories to distributions over actions: $\policy_\timeidx : \historyset_{\timeidx-1} \xrightarrow{} \setofdist(\actset)$. However, not the whole past history is available for the agent in timestep $\timeidx$, formally, $\history_\timeidx = \history_{\timeidx_{\batchidx}}$ for any $t_j < t \leq t_{j+1}$. Despite the batch setting is a property of the environment, we consider this limitation from a policy perspective. With this, we assume that it is not the online agent who works with the batch environment, but the batch policy interacts with the online environment. To distinguish between online and batch policies we will denote the last as $\batchpolicy = (\batchpolicy_{\timeidx})_{1 \leq \timeidx \leq \horizon}$.

\subsection{Additional assumptions}
Before proceeding, we will define a binary relation on a set of policies. We say that the decision rule  $\policy_\timeidx = \policy_\timeidx(\cdot|\history_{\timeidx-1})$ is not worse than the decision rule $\policy_\timeidx^{\prime} = \policy_\timeidx^{\prime}(\cdot|\history_{\timeidx-1})$ (and write $\policy_\timeidx \geq \policy_\timeidx^{\prime}$) if the expected reward under policy $\policy_\timeidx$ is not less than the expected reward under policy $\policy_\timeidx^{\prime}$: 

\begin{equation}
    \label{better policy}
    \sum_{\action \in \actset} \armval_\action \policy_\timeidx(\action|\history_{\timeidx-1}) \geq \sum_{\action \in \actset} \armval_\action \policy_\timeidx^{\prime}(\action|\history_{\timeidx-1}).
\end{equation}

If $\geq$ can be replaced by $>$ in (\ref{better policy}), we say that the decision rule $\policy_\timeidx$ is better than the decision rule $\policy_\timeidx^{\prime}$ (and write $\policy_\timeidx > \policy_\timeidx^{\prime}$). Denote by $\sigma_\timeidx^2 = (\sigma_{\action, \timeidx}^2)_{\action \in \actset}$ the vector of conditional variances of estimated rewards $(\hat{\armval}_{\action, \timeidx})_{\action \in \actset}$ given the history $\history_{\timeidx-1}$:  $\sigma_{\action, \timeidx}^2 = Var (\hat{\armval}_{\action, \timeidx} | \history_{\timeidx-1})$, where $\hat{\armval}_{\action, \timeidx} = \mathbb{E}(\armval_\action | \history_{\timeidx-1})$.

For building up to the proof of the main result, we also need to require some properties of policy $\policy$ that would distinguish "good" policies from the rest. First, we require that if we have two vectors of conditional variances associated with two different histories, then the decision rule based on history with lower variances (over the best arms) is better than the decision rule based on history with higher variances.  

\begin{assumption}
\label{variance_contractions}
    Let $\sigma_\timeidx^2$ and $(\sigma_\timeidx^2)^{\prime}$ be two sets of variances associated with histories $\history_{\timeidx-1}$ and $\history^{\prime}_{\timeidx-1}$, respectively. If  $\sigma_\timeidx^2 \leq_{\env} (\sigma_\timeidx^2)^{\prime}$, then $\policy_\timeidx(\cdot|\history_{\timeidx-1}) \geq \policy_\timeidx(\cdot|\history^{\prime}_{\timeidx-1})$, where $\leq_{\env}$ is an elementwise comparison over the elements $\bestarmidx \in \arg \max_\actset \armval_\action$.
\end{assumption}

Next, we assume that policy $\policy = (\policy_{\timeidx})_{1 \leq \timeidx \leq \horizon}$ \textit{improves over time} if the "rate" of increasing of the regret decreases.

\begin{assumption}
\label{policy_improvement}
    $\frac{\regret_{\horizonn}(\policy)}{\horizonn} > \frac{\regret_{\horizonnn}(\policy)}{\horizonnn}$ for all $\horizonn, \horizonnn$, $1 \leq \horizonn < \horizonnn \leq \horizon$.
\end{assumption}

Intuitively, if an online policy $\policy$ is not good enough (e.g. it makes a lot of suboptimal choices), then an online "short" policy could perform better as it omits these suboptimal choices. Indeed, using Assumption \ref{policy_improvement} for horizons $\numbatches$ and $\horizon$, we have $\frac{\regret_\numbatches(\policy)}{\numbatches} > \frac{\regret_\horizon(\policy)}{\horizon}$, and, therefore, multiplying by $\horizon$ and $\numbatches$ respectively and using $\horizon = \numbatches \batchsize$, we get $b M \regret_\numbatches(\policy)  > M \regret_\horizon(\policy)$. Finally, dividing by $\numbatches$ we get $\batchsize \regret_\numbatches(\policy) > \regret_\horizon(\policy)$.

The following lemma provides properties of a policy that improves over time.
\begin{lemma}
\label{lemma_propeties}
Let $\policy = ( \policy_\timeidx )_{1 \leq \timeidx \leq \horizon}$ be a  policy such that Assumption \ref{policy_improvement} holds. Then,
    \begin{enumerate}
        \item 
        \label{point1}
        $\Bar{\policy}_{\horizonnn} > \Bar{\policy}_{\horizonn}$, where $\Bar{\policy}_{\timeidx} = \frac{\sum_{\timeidxx=1}^{\timeidx} \policy_\timeidxx}{\timeidx}$ is an average decision rule;
        \item 
        \label{point2}
        $\policy_{\timeidx} > \Bar{\policy}_{\timeidx}$ $\forall \timeidx$ such that $1 \leq \timeidx \leq \horizon$,
    \end{enumerate}

where $\policy_\timeidx + \policy_\timeidxx$ is an elementwise addition of two probability vectors for some $\timeidx, \timeidxx$.
\end{lemma}

\begin{proof}
    1. First, we need to show that $\Bar{\policy}_{\timeidx}$ is a decision rule for some $\timeidx$, i.e. $\sum_{\action \in \actset} \Bar{\policy}_{\timeidx}(\action) = 1$ and $\Bar{\policy}_{\timeidx}(\action) \geq 0$ for all $\action \in \actset$. Indeed,
    \begin{equation*}
        \sum_{\action \in \actset} \Bar{\policy}_{\timeidx}(\action) = \sum_{\action \in \actset} \frac{\sum_{\timeidxx=1}^{\timeidx} \policy_\timeidxx (\action)}{\timeidx} = \frac{\sum_{\timeidxx=1}^{\timeidx} \sum_{\action \in \actset} \policy_\timeidxx (\action)}{\timeidx} = \frac{\sum_{\timeidxx=1}^{\timeidx} 1}{\timeidx} = 1.
    \end{equation*}
    
    Since $\policy_\timeidxx (\action) \geq 0$ for all $\action \in \actset$ and for all $1 \leq \timeidxx \leq \timeidx $, $\Bar{\policy}_{\timeidx}(\action) \geq 0$ for all $\action \in \actset$.
    
    Next, we convert $\mathbb{E} [ \totalreward_{\horizon} ]$ into the sum over timesteps and actions:
    \begin{align*}
        \mathbb{E}  [ \totalreward_{\horizon}  ] & = \mathbb{E} \left[ \sum_\timeidx \reward_\timeidx \right] = \mathbb{E} \left[ \sum_{\timeidx} \sum_{\action} \reward_\timeidx \mathbb{I} \{ \Action_\timeidx = \action \} \right] = \mathbb{E} \left[ \sum_\timeidx \sum_\action \mathbb{E} [ \reward_\timeidx \mathbb{I} \{ \Action_\timeidx = \action \} | \Action_\timeidx ] \right] \\
        & = \mathbb{E} \left[ \sum_\timeidx \sum_\action \armval_{\Action_\timeidx} \mathbb{I} \{ \Action_\timeidx = \action \} \right] = \sum_\timeidx \sum_a \armval_\action \mathbb{P}_{\env, \policy} ( \Action_\timeidx = \action ) = \sum_\timeidx \sum_\action \armval_\action \policy_\timeidx ( \action | \history_{\timeidx-1} ).
    \end{align*}
    
    Fix $\horizonn, \horizonnn: \horizonn < \horizonnn$. We have 
    $\frac{\regret_{\horizonn}(\policy)}{\horizonn} > \frac{\regret_{\horizonnn}(\policy)}{\horizonnn}$. Expressing the regret by its definition, one can get $\frac{\horizonn \bestarm - \mathbb{E} [\totalreward_{\horizonn}]}{\horizonn} > \frac{\horizonnn \bestarm - \mathbb{E} [\totalreward_{\horizonnn}]}{\horizonnn}$, and hence $\frac{\mathbb{E} [\totalreward_{\horizonnn}]}{\horizonnn} - \frac{\mathbb{E} [\totalreward_{\horizonn}]}{\horizonn} > 0$.
    
    Finally,
    
    \begin{equation*}
        \frac{\mathbb{E} [\totalreward_{\horizonnn}]}{\horizonnn} - \frac{\mathbb{E} [\totalreward_{\horizonn}]}{\horizonn} = \frac{\sum_{\timeidx=1}^{\horizonnn} \sum_\action \armval_\action \policy_\timeidx ( \action | \history_{\timeidx-1} )}{\horizonnn} - \frac{\sum_{\timeidx=1}^{\horizonn} \sum_\action \armval_\action \policy_\timeidx ( \action | \history_{\timeidx-1} )}{\horizonn} > 0.
    \end{equation*}
    
    The result is completed by rearranging the sums and using the definition of $\Bar{\policy}_{\horizonn}, \Bar{\policy}_{\horizonnn}$.
    
    2. For $\timeidx<\horizon$ we have $ \frac{ \regret_{\timeidx}(\policy)}{\timeidx} > \frac{ \regret_{\timeidx+1}(\policy)}{\timeidx+1} $. By subtracting $\frac{ \regret_{\timeidx+1}(\policy)}{\timeidx}$ from both sides we get:
    \begin{align*}
        \frac{ \regret_{\timeidx}(\policy) - \regret_{\timeidx+1}(\policy)}{t} & > \frac{\timeidx \regret_{\timeidx+1}(\policy) - (\timeidx+1) \regret_{\timeidx+1}(\policy)}{\timeidx(\timeidx+1)}, \\
        \frac{-(\bestarm - \reward_{\timeidx+1})}{\timeidx} & > \frac{-\regret_{\timeidx+1}(\policy)}{(\timeidx+1)\timeidx}, \\
        \bestarm - \reward_{\timeidx+1} & < \frac{\regret_{\timeidx+1}(\policy)}{\timeidx+1}, \\
        \bestarm - \sum_\action \armval_\action \policy_{\timeidx+1}(\action) & < \bestarm - \sum_\action \armval_\action \Bar{\policy}_{\timeidx+1}(\action), \\
        \sum_\action \armval_\action \policy_{\timeidx+1}(\action) & > \sum_\action \armval_\action \Bar{\policy}_{\timeidx+1}(\action).
    \end{align*}
    
    Here, in forth step we used that $\frac{\sum_{\action \in \actset}\sum_{\timeidxx=1}^{\timeidx+1} \policy_\timeidxx (\action)}{\timeidx+1} = \sum_{\action \in \actset} \Bar{\policy}_{\timeidx+1}(\action)$.
\end{proof}

\section{Lower and upper bounds}
We next provide lower and upper bounds on the best achievable performance. Intuitively, if policy $\policy$ is "good enough" then it should be better than policy $\batchpolicy$ because it makes its decisions ("good decisions") having the maximum available information about the environment. In some sense, we can consider a batch policy $\batchpolicy$ as a "slow" version of policy $\policy$. Nevertheless, batch policy $\batchpolicy$ is a batch specification of "good" policy $\policy$, and we can expect that it performs relatively well. The following theorem formalizes the described above intuition.

\begin{theorem}
\label{thm1}
    Let $\pi^b$ be a batch specification of a given policy $\pi$ and $\numbatches = \frac{\horizon}{\batchsize}$. Suppose that assumptions \ref{variance_contractions} and \ref{policy_improvement} hold. Then, for $b>1$,
    
    \begin{equation}
    \label{main_thm}
        \regret_\horizon(\policy) < \regret_\horizon(\batchpolicy) \leq \batchsize \regret_\numbatches(\policy).
    \end{equation}
\end{theorem}

We can consider the term $\batchsize \regret_\numbatches(\policy)$ as $b$ similar online independent agents operating in the same environment by following policy $\policy$ but over a shorter horizon $M$. However, we can also think of it as the performance of a batch policy. Indeed, imagine that this is one agent that deliberately repeats each step $b$ times instead of immediately updating its beliefs and proceeding to the next round as in an online manner. Then, after $b$ repetitions (after a batch ends), it updates its beliefs using only the first reward from the previous batch. So, while this policy could perform as an online policy within these repetitions, it pretends that rewards are not observable and acts using the outdated history (just like a batch policy does). For notational simplicity, we denote the online "short" policy as $\dumnpolicy$. Thus, the theorem allows to benchmark any batch policy in terms of its online analogs. Note that there is no any behavioral difference between the original policy $\pi$ and the "short" policy $\dumnpolicy$.

\section{Regret analysis}
\label{theory}

\subsection{Sketch of the proof}
\label{sketch}
First, we analyze the performance of the three policies $\policy$, $\batchpolicy$, and $\dumnpolicy$ within a batch. We show that if at the beginning of a batch the online policy $\policy$ is not worse than the batch policy $\batchpolicy$, which in turn is not worse than the online "short" policy $\dumnpolicy$, then, by the end of the batch, we get the corresponding regrets being related reversely. Then, by Assumption \ref{variance_contractions} we establish that the transition from the end of one batch to the beginning of the next batch retains the relation between policies.

\textbf{Within batch.} Lets consider the timestep $\timeidx_{\batchidx}$ which is the beginning of the $j$-th batch. Assume that in timestep $\timeidx_{\batchidx}$ the decision rule of the online policy is not worse that than the decision rule of the batch policy, which is not worse than the decision rule of the online "short" policy, $\policy_{\timeidx_\batchidx} \geq \batchpolicy_{\timeidx_\batchidx} \geq \dumnpolicy_{\timeidx_\batchidx}$. Since the online policy $\policy$ improves within a batch (by Assumption \ref{policy_improvement}) and both the batch policy $\batchpolicy$ and online "short" policy $\dumnpolicy$ remain the same within a batch (by the definition of batch policy), by applying Lemma \ref{lemma_propeties} we have $\policy_{\timeidx} > \batchpolicy_{\timeidx} \geq \dumnpolicy_{\timeidx}$ for any timestep $\timeidx_{\batchidx} < \timeidx < \timeidx_{\batchidx+1}$. So, starting at the beginning of a batch with $\policy_{\timeidx_\batchidx} \geq \batchpolicy_{\timeidx_\batchidx} \geq \dumnpolicy_{\timeidx_\batchidx}$ leads us to $\policy_{\timeidx_{\batchidx+1}-1} > \batchpolicy_{\timeidx_{\batchidx+1}-1} \geq \dumnpolicy_{\timeidx_{\batchidx+1}-1}$ at the end of the batch.

\textbf{Between batches.} 
Let $\history_{\timeidx_{\batchidx}-1}$, $\history^\batchsize_{\timeidx_{\batchidx}-1}$, $\history^{\prime}_{\timeidx_{\batchidx}-1}$ be histories collected by policies $\policy$, $\batchpolicy$, $\dumnpolicy$ by the timestep $\timeidx_{\batchidx}$, correspondingly. Let $\sigma_{\timeidx_{\batchidx}}^2$, $(\sigma_{\timeidx_{\batchidx}}^2)^{\batchsize}$, $(\sigma_{\timeidx_{\batchidx}}^2)^{\prime}$ be the corresponding conditional variances. Transiting from timestep $\timeidx_{\batchidx}-1$ to timestep $\timeidx_{\batchidx}$ leads us to a new batch, and all policies are being updated in this step. To show that the relation between batch policy and its bounds holds, we need to analyze how much information was gained by all three policies, namely, we need to show that $\sigma_{\timeidx_{\batchidx}}^2 \leq_{\env} (\sigma_{\timeidx_{\batchidx}}^2)^{\batchsize} \leq_{\env} (\sigma_{\timeidx_{\batchidx}}^2)^{\prime}$. Then, we can just apply Assumption $\ref{variance_contractions}$.

We now exploit the fact that policies $\batchpolicy$ and $\dumnpolicy$ are specifications of online policy $\policy$. This means that whatever exploration-exploitation scheme is implemented in policy $\policy$, policies $\batchpolicy$ and $\dumnpolicy$ mimic the same scheme but in slower manner. In other words, ones policy $\policy$ reaches a particular configuration $\sigma^2$, policies $\batchpolicy$ and $\dumnpolicy$ reach the same configurations $(\sigma^2)^{\batchsize}$, $(\sigma^2)^{\prime}$ much later in time. Similarly, policy $\dumnpolicy$ is a slower version of policy $\batchpolicy$. As a result, for an arbitrary timestep $t$ we have $\sigma_{\timeidx_{\batchidx}}^2 \leq_{\env} (\sigma_{\timeidx_{\batchidx}}^2)^{\batchsize} \leq_{\env} (\sigma_{\timeidx_{\batchidx}}^2)^{\prime}$.

\subsection{Proof of Theorem \ref{main_thm}}
\label{proof}
The sketch of the proof works for an arbitrary number of actions K but, to bring it to life, one needs to show that online policy cannot pull as many "bad" arms as batch and "short" online policies do. Below we provide the formal proof for the case $\numofactions = 2$ and we leave for future work the case $\numofactions>2$.

\textbf{Step 1 (Within batch).} Fix $\batchidx \geq 1$. Let $\policy_{\timeidx_\batchidx} \geq \batchpolicy_{\timeidx_\batchidx} \geq \dumnpolicy_{\timeidx_\batchidx}$. Define an average decision rule between timesteps $\timeidx_1$ and $\timeidx_2$ as $\Bar{\policy}_{\timeidx_1, \timeidx_2} = \frac{\sum_{\timeidxx=\timeidx_1}^{\timeidx_2-1} \policy_\timeidxx}{\timeidx_2 - 1 - \timeidx_1 }$; and an average decision rule in batch $\batchidx$ as $\avpolicy{\batchidx} = \frac{\sum_{\timeidxx=\timeidx_\batchidx}^{\timeidx_{\batchidx+1}-1} \policy_\timeidxx}{b }$. By the end of batch $j$ have:

\begin{align*}
    \Bar{\policy}_{\timeidx_\batchidx,\timeidx} \stackrel{(a)}{>} \Bar{\policy}_{\timeidx_\batchidx, \timeidx_\batchidx + 1} = \policy_{\timeidx_\batchidx} \geq & \batchpolicy_{\timeidx_\batchidx} \stackrel{(b)}{=} \batchpolicy_{\timeidx} \text{,  and} \\
    \\
    & \batchpolicy_{\timeidx_\batchidx} \geq \dumnpolicy_{\timeidx_\batchidx} \stackrel{(c)}{=} \dumnpolicy_{\timeidx},
\end{align*}

for any timestep $\timeidx_{\batchidx} < \timeidx < \timeidx_{\batchidx+1}$, where $(a)$ follows from Lemma \ref{lemma_propeties} (\ref{point1}); and $(b)$ and $(c)$ hold by the definition of batch policy. Thus, starting with $\policy_{\timeidx_\batchidx} \geq \batchpolicy_{\timeidx_\batchidx} \geq \dumnpolicy_{\timeidx_\batchidx}$ at the beginning of batch $\batchidx$ leads us to $\avpolicy{\batchidx} > \avbatchpolicy{\batchidx} \geq \avdumnpolicy{\batchidx}$. Moreover, by Lemma \ref{lemma_propeties} (\ref{point2}), we have $\policy_{\timeidx_{\batchidx + 1} - 1} > \batchpolicy_{\timeidx_{\batchidx + 1} - 1} \geq \dumnpolicy_{\timeidx_{\batchidx + 1} - 1}$. 

\textbf{Step 2 (Between batches).} Assume that $\numofactions = 2$ and fix $\batchidx \geq 1$. Let $\avpolicy{{l}} > \avbatchpolicy{{l}} \geq \avdumnpolicy{{l}}$ for any batch $1 \leq l < j$. Let $\history_{\timeidx_{\batchidx}-1}$, $\history^\batchsize_{\timeidx_{\batchidx}-1}$, $\history^{\prime}_{\timeidx_{\batchidx}-1}$ be histories collected by policies $\policy$, $\batchpolicy$, $\dumnpolicy$ by the timestep $\timeidx_{\batchidx}$, correspondingly. Let $\sigma_{\timeidx_{\batchidx}}^2$, $(\sigma_{\timeidx_{\batchidx}}^2)^{\batchsize}$, $(\sigma_{\timeidx_{\batchidx}}^2)^{\prime}$ be the corresponding conditional variances. Define a number of times we have received a reward from arm $\action$ \footnote{Usually, it is defined as  a number of times action $\action$ has been played but, since the policy $\dumnpolicy$ plays more arms than receives rewards, we define it that way.} in batch $\batchidx$ by policy $\policy$ as $T_a(\policy, \batchidx) = \sum_{\timeidxx=\timeidx_\batchidx}^{\timeidx_{\batchidx+1}-1} \mathbb{I} \{\Action_\timeidx = \action\}$. Note that $\mathbb{E} [T_a(\policy, \batchidx)] =  b \cdot \avpolicy{{\batchidx}}(a) $. Since $\numofactions = 2$, $\avpolicy{{l}} > \avbatchpolicy{{l}} \geq \avdumnpolicy{{l}}$ implies $\avpolicy{{l}}(\bestarmidx) > \avbatchpolicy{{l}}(\bestarmidx) \geq \avdumnpolicy{{l}}(\bestarmidx)$ for $1 \leq l < j$. Hence, $\mathbb{E}[T_{\bestarmidx}(\policy, l)] > \mathbb{E}[T_{\bestarmidx}(\batchpolicy, l)] \geq \mathbb{E}[T_{\bestarmidx}(\dumnpolicy, l)]$ for $1 \leq l < j$ and, therefore, $\sum_l \mathbb{E}[T_{\bestarmidx}(\policy, l)] > \sum_l \mathbb{E}[T_{\bestarmidx}(\batchpolicy, l)] \geq \sum_l \mathbb{E}[T_{\bestarmidx}(\dumnpolicy, l)]$. In such case, we have $\sigma_{\timeidx_{\batchidx}}^2 \leq_{\env} (\sigma_{\timeidx_{\batchidx}}^2)^{\batchsize} \leq_{\env} (\sigma_{\timeidx_{\batchidx}}^2)^{\prime}$, since the variance is inversely proportional to the number of rewards has been received. By applying Assumption \ref{variance_contractions} we have that $\policy_{\timeidx_{\batchidx}} > \batchpolicy_{\timeidx_{\batchidx}} \geq \dumnpolicy_{\timeidx_{\batchidx}}$.

\textbf{Step 3 (Regret throughout the horizon).} We assume that the interaction begins with the online policy, batch policy and "short" online policy being equal to each other: $\policy_{1} = \batchpolicy_{1} = \dumnpolicy_{1}$. Then, from Step 1, by the end of the first batch we have $\avpolicy{1} > \avbatchpolicy{1} \geq \avdumnpolicy{1}$ and $\policy_{\timeidx_{2} - 1} > \batchpolicy_{\timeidx_{2} - 1} \geq \dumnpolicy_{\timeidx_{2} - 1}$. Next, from Step 2, the transition to the second batch retains the relation between policies: $\policy_{\timeidx_{2}} > \batchpolicy_{\timeidx_{2}} \geq \dumnpolicy_{\timeidx_{2}}$; and so on. Finally, summing over $\numbatches=\frac{\horizon}{\batchsize}$ batches we have:

\begin{align*}
    \regret_\horizon(\policy) & = \horizon \mu^* - \mathbb{E} \left[ \sum_{\timeidx=1}^{\horizon} \reward_{\timeidx} \right] \\
    & = \sum_{\batchidx=1}^{\numbatches} \left( \batchsize \mu^* - \mathbb{E} \left[ \sum_{\timeidx=\timeidx_\batchidx}^{\timeidx_{\batchidx+1}-1} \reward_{\timeidx} \right] \right) \\
    & = \sum_{\batchidx=1}^{\numbatches} \left( \batchsize \mu^* - \batchsize \sum_a\avpolicy{{j}}(\action) \armval_\action \right) &&< \sum_{\batchidx=1}^{\numbatches} \left( \batchsize \mu^* - \batchsize \sum_a\avbatchpolicy{{j}}(\action) \armval_\action \right) = \regret_\horizon(\batchpolicy) \\
    & =  \batchsize \sum_{\batchidx=1}^{\numbatches} \left( \mu^* -  \sum_a\avbatchpolicy{{j}}(\action) \armval_\action \right) &&\leq \batchsize \sum_{\batchidx=1}^{\numbatches} \left( \mu^* -  \sum_a\avdumnpolicy{{j}}(\action) \armval_\action \right) = \batchsize \regret_\numbatches(\policy).
\end{align*}

This concludes the proof.

\section{Empirical analysis}
\label{empirical}
We perform experiments on two different applications: on simulated stochastic bandit environments; and on a contextual bandit environment, learned from the logged data in an online marketing campaign. We examine the effect of batch learning on Thompson Sampling (TS) and Upper Confidence Bound (UCB) policies for the stochastic problems, and linear Thompson Sampling (LinTS) \citep{NIPS2011_e53a0a29} and linear Upper Confidence Bound (LinUCB) \citep{LinUCB_2010} for the contextual problem.

\textbf{Simulated environments.} We present some simulation result for the Bernoulli bandit problem. We create 3 environments with different mean rewards for each $\numofactions = 2$ and $\numofactions = 4$. Table \ref{env-table} summarizes information about vectors of mean rewards used in the environments.

\begin{table}
  \caption{Description of the environments.}
  \label{env-table}
  \centering
  \begin{tabular}{llll}
    \toprule
    \multicolumn{2}{c}{2-arm environments} & \multicolumn{2}{c}{4-arm environments} \\
    \midrule
    Name     & Rewards     & Name & Rewards  \\
    \cmidrule(r){1-2} 
    \cmidrule(r){3-4}
    $env1$ & [0.7, 0.5]  & $env4$ &   [0.35, 0.18, 0.47, 0.61]   \\
    $env2$ & [0.7, 0.4]  & $env5$ &   [0.40, 0.75, 0.57, 0.49]   \\
    $env3$ & [0.7, 0.1]  & $env6$ &   [0.70, 0.50, 0.30, 0.10]   \\
    \bottomrule
  \end{tabular}
\end{table}

\textbf{Real data.} We also consider batch learning in a personalized marketing campaign on the KPN \footnote{a dutch telecommunications provider} logged dataset. KPN has recently used 3 different campaigns that all aimed to sell an extra broadband subscription to their customers. In the current dataset, a randomly selected set of customers received one of the three campaigns randomly. The data contains a sample of a campaign selection from October 2019 until June 2020 combined with customer information.  We adopt an unbiased offline evaluation method of \citep{Li_2011} to compare various bandit algorithms and batch size values. We use conversion rate (CR) as the metric of interest, which is defined as the ratio between the number of successful interactions and the total number of interactions. To protect business-sensitive information, we only report relative conversion rate, therefore Figure \ref{cr-plot} demonstrates the CR returned by the off-policy evaluation algorithm with hidden y-axis values.

\textbf{Results.} \footnote{The source code of the experiments can be found in \url{https://github.com/danilprov/batch-bandits}.} Figure \ref{regret-plot} shows the regret of two algorithms in six environments across batch size. As expected, the regret has an upward trend as batch size increases. It can be seen that, for both algorithms, the impact of batch learning depends on environment parameters: the bigger a suboptimality gap, the stronger the impact of batching.

While the exact behavior of the batch learning is still not understood (and it pretty much depends on an algorithm itself rather than the experiment design), we can see a clear difference between a randomised policy and a deterministic one. Indeed, TS turns out to be more robust to the impact of batching, whereas UCB algorithm fluctuates considerably as batch size increases. The results for the real data confirms this fact as well: from Figure \ref{cr-plot} we observe that the impact of batching is milder for randomised policy (LinTS) than for deterministic policy (LinUCB) in contextual problem.

It is important to note that both experiments demonstrate results consistent with the theoretical analysis conducted in Section \ref{theory}. As the upper bound in Theorem \ref{thm1} suggests, the performance metric reacts evenly to the increasing/decreasing of the batch size. Thus, the conducted analysis guarantees that the choice of the batch size should be rather based on computational capabilities or other features of the problem. Finally, although it was not an initial goal of the study, we recommend to resort to randomised policies when it becomes necessary to learn in batches.

\begin{figure}
\centering     
\subfigure[2-arm environments performance.]{\label{fig:b}\includegraphics[width=60mm]{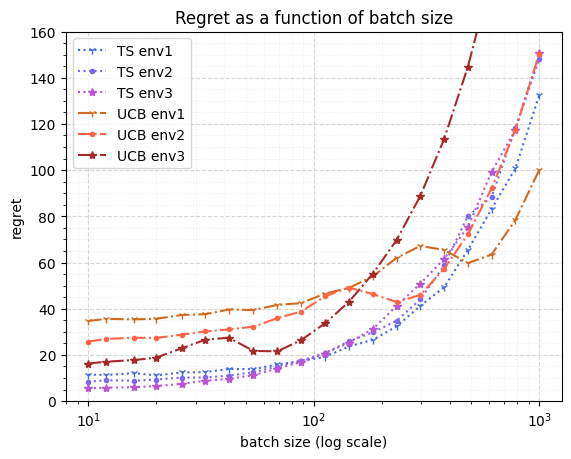}}
\subfigure[4-arm environments performance.]{\label{fig:a}\includegraphics[width=60mm]{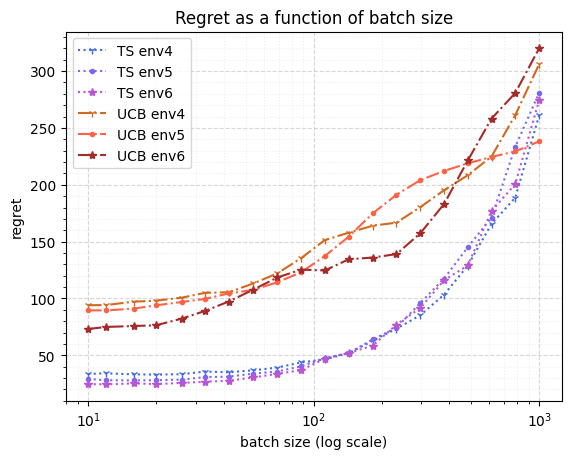}}
\caption{Empirical regret performance by batches. The plots are averaged over 500 repetitions.}
\label{regret-plot}
\end{figure}

\begin{figure}
\centering     
{\includegraphics[width=60mm]{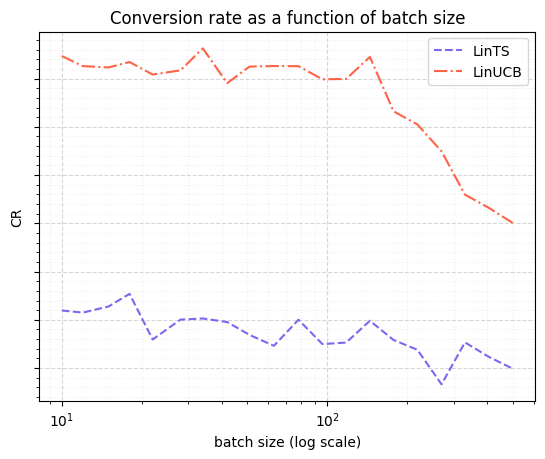}}
\caption{Empirical conversion rate by batches. The plots are averaged over 20 repetitions.}
\label{cr-plot}
\end{figure}

\section{Discussion and further directions}
Batch learning is an integral part of any learning system, and the understanding of its impact is of a first magnitude. In this study, we have presented a systematic approach for batch learning in stochastic bandits. Formulating the problem from a more practically relevant perspective, we have shown the true effect of batch learning by conducting a comprehensive theoretical analysis, which is confirmed by our strong empirical results.

Practically speaking, we have investigated one component of the performance-computational cost trade-off and demonstrated that it deteriorates gradually depending on the batch size. Thus, in order to find a suitable batch size, practitioners should take the necessary steps to estimate the second component (computational costs) based on computational capabilities.

Concerning future work, the first obvious step is to extend proof of the Theorem \ref{thm1} to $\numofactions > 2$  case. 
It is also crucial to identify a policy class $\Pi$ satisfying Assumption \ref{variance_contractions}. 
Further natural direction is to extend the provided theoretical guarantees to a wider class of MAB problems (e.g., contextual and non-stationary cases). While our formulation focuses on stochastic bandits, our established bounds should extend naturally for contextual problems as well. In fact, the main technical contribution does not rely on any properties of stochastic bandits. Finally, as we mentioned earlier, randomised policies seem to be more robust to the batch learning than deterministic policies. An interesting future direction, therefore, is to provide theoretical comparison of batch learning impact on randomised and deterministic policies separately.

\acksection
We would like to thank the anonymous reviewers for their thoughtful and helpful suggestions and comments. This project is partially financed by the Dutch Research Council (NWO) and the ICAI initiative in collaboration with KPN, the Netherlands.

\bibliography{main}
\bibliographystyle{plainnat}

\newpage
\iftoggle{long-version}{
\appendix

\section{Appendix}

\subsection{Relation between lower and upper bounds}
\label{ap_low_up_bound}

In this section, we discuss how reasonable Assumption \ref{policy_improvement} is. Intuitively, if an online policy $\policy$ is not good enough (e.g. it makes a lot of suboptimal choices), then an online "short" policy could perform better as it omits these suboptimal choices. Indeed, using Assumption \ref{policy_improvement}, we can show that $\batchsize \regret_\numbatches(\policy) > \regret_\horizon(\policy)$:

\begin{align*}
    \frac{\regret_\numbatches(\policy)}{\numbatches} & > \frac{\regret_\horizon(\policy)}{\horizon}, \\
    n \regret_\numbatches(\policy) & > M \regret_\horizon(\policy), \\
    b M \regret_\numbatches(\policy) & > M \regret_\horizon(\policy).
\end{align*}

Since $n > M$, the above chain of inequalities suggests that online "short" policy performs worse than online policy.

\subsection{Proof of lemma \ref{lemma_propeties}}
\begin{enumerate}
    \item First, we need to show that $\Bar{\policy}_{\timeidx}$ is a decision rule for some $\timeidx$, i.e. $\sum_{\action \in \actset} \Bar{\policy}_{\timeidx}(\action) = 1$ and $\Bar{\policy}_{\timeidx}(\action) \geq 0$ for all $\action \in \actset$. Indeed,
    \begin{equation*}
        \sum_{\action \in \actset} \Bar{\policy}_{\timeidx}(\action) = \sum_{\action \in \actset} \frac{\sum_{\timeidxx=1}^{\timeidx} \policy_\timeidxx (\action)}{\timeidx} = \frac{\sum_{\timeidxx=1}^{\timeidx} \sum_{\action \in \actset} \policy_\timeidxx (\action)}{\timeidx} = \frac{\sum_{\timeidxx=1}^{\timeidx} 1}{\timeidx} = 1.
    \end{equation*}
    
    Since $\policy_\timeidxx (\action) \geq 0$ for all $\action \in \actset$ and for all $1 \leq \timeidxx \leq \timeidx $, $\Bar{\policy}_{\timeidx}(\action) \geq 0$ for all $\action \in \actset$.
    
    Next, we convert $\mathbb{E} [ \totalreward_{\horizon} ]$ into the sum over timesteps and actions:
    \begin{align*}
        \mathbb{E}  [ \totalreward_{\horizon}  ] & = \mathbb{E} \big [ \sum_\timeidx \reward_\timeidx \big ] = \mathbb{E} \big [ \sum_{\timeidx} \sum_{\action} \reward_\timeidx \mathbb{I} \{ \Action_\timeidx = \action \} \big ] = \mathbb{E} \big [ \sum_\timeidx \sum_\action \mathbb{E} [ \reward_\timeidx \mathbb{I} \{ \Action_\timeidx = \action \} | \Action_\timeidx ] \big ] \\
        & = \mathbb{E} \big [ \sum_\timeidx \sum_\action \armval_{\Action_\timeidx} \mathbb{I} \{ \Action_\timeidx = \action \} \big ] = \sum_\timeidx \sum_a \armval_\action \mathbb{P}_{\env, \policy} ( \Action_\timeidx = \action ) = \sum_\timeidx \sum_\action \armval_\action \policy_\timeidx ( \action | \history_{\timeidx-1} ).
    \end{align*}
    
    Fix $\horizonn, \horizonnn: \horizonn < \horizonnn$. We have 
    $\frac{\regret_{\horizonn}(\policy)}{\horizonn} > \frac{\regret_{\horizonnn}(\policy)}{\horizonnn}$. Expressing the regret by its definition, one can get $\frac{\horizonn \bestarm - \mathbb{E} [\totalreward_{\horizonn}]}{\horizonn} > \frac{\horizonnn \bestarm - \mathbb{E} [\totalreward_{\horizonnn}]}{\horizonnn}$, and hence $\frac{\mathbb{E} [\totalreward_{\horizonnn}]}{\horizonnn} - \frac{\mathbb{E} [\totalreward_{\horizonn}]}{\horizonn} > 0$.
    
    Finally,
    
    \begin{equation*}
        \frac{\mathbb{E} [\totalreward_{\horizonnn}]}{\horizonnn} - \frac{\mathbb{E} [\totalreward_{\horizonn}]}{\horizonn} = \frac{\sum_{\timeidx=1}^{\horizonnn} \sum_\action \armval_\action \policy_\timeidx ( \action | \history_{\timeidx-1} )}{\horizonnn} - \frac{\sum_{\timeidx=1}^{\horizonn} \sum_\action \armval_\action \policy_\timeidx ( \action | \history_{\timeidx-1} )}{\horizonn} > 0.
    \end{equation*}
    
    The result is completed by rearranging the sums and using the definition of $\Bar{\policy}_{\horizonn}, \Bar{\policy}_{\horizonnn}$.
    
    \item For $\timeidx<\horizon$ we have $ \frac{ \regret_{\timeidx}(\policy)}{\timeidx} > \frac{ \regret_{\timeidx+1}(\policy)}{\timeidx+1} $. By subtracting $\frac{ \regret_{\timeidx+1}(\policy)}{\timeidx}$ from both sides we get:
    \begin{align*}
        \frac{ \regret_{\timeidx}(\policy) - \regret_{\timeidx+1}(\policy)}{t} & > \frac{\timeidx \regret_{\timeidx+1}(\policy) - (\timeidx+1) \regret_{\timeidx+1}(\policy)}{\timeidx(\timeidx+1)} \\
        \frac{-(\bestarm - \reward_{\timeidx+1})}{\timeidx} & > \frac{-\regret_{\timeidx+1}(\policy)}{(\timeidx+1)\timeidx} \\
        \bestarm - \reward_{\timeidx+1} & < \frac{\regret_{\timeidx+1}(\policy)}{\timeidx+1} \\
        \bestarm - \sum_\action \armval_\action \policy_{\timeidx+1}(\action) & < \bestarm - \sum_\action \armval_\action \Bar{\policy}_{\timeidx+1}(\action) \\
        \sum_\action \armval_\action \policy_{\timeidx+1}(\action) & > \sum_\action \armval_\action \Bar{\policy}_{\timeidx+1}(\action)
    \end{align*}
    
    Here, in forth step we used that $\frac{\sum_{\action \in \actset}\sum_{\timeidxx=1}^{\timeidx+1} \policy_\timeidxx (\action)}{\timeidx+1} = \sum_{\action \in \actset} \Bar{\policy}_{\timeidx+1}(\action)$.
\end{enumerate}
}{}
\end{document}